\documentclass[acmlarge]{modified}

\AtBeginDocument{%
  \providecommand\BibTeX{{%
    \normalfont B\kern-0.5em{\scshape i\kern-0.25em b}\kern-0.8em\TeX}}}

\newif\ifTR
\TRtrue

\begin{document}

\title{How many moments does MMD compare?}

\author{Rustem Takhanov}
\email{rustem.takhanov@nu.edu.kz}
\orcid{0000-0001-7405-8254}
\authornotemark[1]
\affiliation{%
  \institution{School of Sciences and Humanities}
  \streetaddress{53 Kabanbay Batyr Ave}
  \city{Nur-Sultan city}
  \country{Republic of Kazakhstan}
  \postcode{010000}
}

\renewcommand{\shortauthors}{Takhanov}

\begin{abstract}
We present a new way of study of Mercer kernels, by corresponding to a special kernel $K$ a  pseudo-differential operator $p({\mathbf x}, D)$ such that $\mathcal{F} p({\mathbf x}, D)^\dag p({\mathbf x}, D) \mathcal{F}^{-1}$ acts on smooth functions in the same way as an integral operator associated with $K$ (where $\mathcal{F}$ is the Fourier transform). We show that kernels defined by pseudo-differential operators are able to approximate uniformly any continuous Mercer kernel on a compact set.

The symbol $p({\mathbf x}, {\mathbf y})$ encapsulates a lot of useful information about the structure of the Maximum Mean Discrepancy distance defined by the kernel $K$. 
We approximate $p({\mathbf x}, {\mathbf y})$ with the sum of the first $r$ terms of the Singular Value Decomposition of $p$, denoted by $p_r({\mathbf x}, {\mathbf y})$.
If ordered singular values of the integral operator associated with $p({\mathbf x}, {\mathbf y})$ die down rapidly, the MMD distance defined by the new symbol $p_r$ differs from the initial one only slightly.
Moreover, the new MMD distance can be interpreted as an aggregated result of comparing $r$ local moments of two probability distributions.

The latter results holds under the condition that right singular vectors of the integral operator associated with $p$ are uniformly bounded. But even if this is not satisfied we can still hold that the Hilbert-Schmidt distance between $p$ and $p_r$ vanishes.
Thus, we report an interesting phenomenon: the MMD distance measures the difference of two probability distributions with respect to a certain number of local moments, $r^\ast$, and this number $r^\ast$ depends on the speed with which singular values of $p$ die down.
\end{abstract}


\maketitle

\section{Introduction}
Recovering a distribution from a high-dimensional dataset, either in the form of a multivariate probability density function or in the form of a sampling model, is the key task in the plethora of data science applications. 
The modern approach to that task is based on minimizing a distance function between the so called empirical distribution $\mu_{\rm emp}$, i.e. the probabilistic measure concentrated in data points, and a parameterized model distribution, $\mu_\theta$. Parameters of the model, $\theta$, are the arguments over which the minimization is performed. Many  distance functions are used for this purpose, including the Kullback-Leibler divergence, the Wasserstein distance~\cite{pmlr-v70-arjovsky17a}, the $f$-divergence~\cite{NIPS2016_cedebb6e} and many others. Kernel methods, that are ubiquitous in machine learning, have also been applied to the problem. A key kernel-based distance function in the generative modeling is the Maximum Mean Discrepancy (MMD) metric. The MMD distance is induced by the following inner product between probability measures $\mu$ and $\nu$:
$$
\langle \mu, \nu\rangle_K = \int_{{\mathbb R}^n\times {\mathbb R}^n} K({\mathbf x}, {\mathbf y}) d\mu ({\mathbf x}) d\nu ({\mathbf y})
$$
where $K: {\mathbb R}^n\times {\mathbb R}^n \to {\mathbb R}$ is a Mercer kernel. In other words, MMD is defined by:
$$
{\rm MMD}_K(\mu, \nu)^2 = \langle \mu, \mu\rangle_K+\langle \nu, \nu\rangle_K - 2\langle \mu, \nu\rangle_K
$$
and the corresponding generative modeling task is:
\begin{equation}\label{main-problem}
\min_{\theta}{\rm MMD}_K(\mu_{\rm emp}, \mu_{\theta})
\end{equation}
where, usually, $\mu_\theta$ is given in the form of a sampler. 

Initial experiments with the MMD based generative modeling, the so called moment matching networks (GMMN)~\cite{gmmn}, demonstrated the superiority of more computationally costly generative adversarial networks (GANs)~\cite{Goodfellow}. A weakness of GMMN in various applications is that MMD is strongly dependant on the underlying kernel $K$. In an attemp to overcome this problem, MMD GANs were introduced~\cite{MMD-GAN,binkowski2018demystifying}, with the GAN style architecture that assumes training of the so called critic network. In MMD GANs the kernel that defines MMD depends on paramaters. The role of the critic  is to tune parameters of the kernel.

In fact, there is a connection between GANs and the MMD based networks that follows from a well-known representation: $${\rm MMD}_K(\mu, \nu) = \sup\limits_{f\in \mathcal{H}_K: ||f||_{\mathcal{H}_K}\leq 1} {\mathbb E}_{X\sim \mu}f(X)-{\mathbb E}_{Y\sim \nu}f(Y),$$ where $\mathcal{H}_K$ is RKHS of $K$~\cite{KernelMean,HilbertSpaceEmbeddings}. Thus, the task~\eqref{main-problem} can be represented in the minimax form, and this allows to treat MMD based networks as GANs. Unfortunately, this representation does not give a direct insight into the structure of MMD, because the critic's space of functions, i.e. a unit ball in $\mathcal{H}_K$, is a nontrivial mathematical object. 

Thus, an important general question is: if we change something in $K$ (e.g. the Gaussian kernel's bandwidth or a degree of a polynomial kernel) how will it affect the structure of the MMD distance? 
To answer this question we introduce a family of kernels defined by the so called pseudo-differential operators~\cite{Taylor2017}. Pseudo-differential operators are a special kind linear operators between  spaces of smooth functions nicely behaving at infinity. Every such operator ${\mathcal L}$ is defined by a function $p: {\mathbb R}^n\times {\mathbb R}^n\to {\mathbb C}$, which is called a symbol of operator. Their relationship is usually denoted as ${\mathcal L} = p({\mathbf x}, D)$. Then we say that a kernel $K$ is defined by ${\mathcal L}$ if $\langle {\mathcal L}u, {\mathcal L}v\rangle_{L_2} \propto \langle \hat{u}, {\rm O}_{K}\hat{v}\rangle_{L_2}$ where $\hat{u}$ denotes the Fourier transform of $u$ and ${\rm O}_{K}u({\mathbf x})=\int_{{\mathbb R}^n} K({\mathbf x},{\mathbf y})u({\mathbf y})d {\mathbf y}$ is an integral operator associated with $K$.

It turns out that many popular Mercer kernels (Gaussian, Poisson, Abel etc) can be given in this way. Moreover, this family of kernels is broad enough to approximate any continuous Mercer kernel $K: \Omega\times \Omega\to {\mathbb R}$ on a compact set $\Omega\subset {\mathbb R}^n$.

For kernels defined by pseudo-differential operators, the MMD distance can be naturally represented as the supremum over a managable unit ball in $L_2({\mathbb R}^n)$, i.e. not a unit ball in $\mathcal{H}_K$. 
This allows us to study the structure of a loss function that the critic maximizes during a computation of the MMD distance. 

For a pseudo-differential operator ${\mathcal L} = p({\mathbf x}, D)$ we observe the following: if the rank of ${\rm O}_{p}$, i.e. the dimension of the range of ${\rm O}_{p}$, equals $r$, then the MMD distance associated with ${\mathcal L}$ can be understood as an aggregated result of a comparison of $r$ ``local'' moments of two distributions. Motivated by this observation, we study the general situation, i.e. when the rank of ${\rm O}_{p}$ is unbounded and ${\rm O}_{p}$ needs to be approximated by another operator of finite rank. This setting reminds us of the principal component analysis of an infinite dimensional operator ${\rm O}_{p}$. 

The approximation error for pseudo-differential operators that we introduced can be naturally measured by two types of norms. The first norm is the operator norm in the space of bounded linear operators from $L_2({\mathbb R}^n)$ to $L_\infty({\mathbb R}^n)$, denoted $||\cdot||_{2,\infty}$. It is natural in a sense that two symbols $p_1, p_2$ with small $||{\rm O}_{p_1}-{\rm O}_{p_2}||_{2,\infty}$ define MMD distances that differ only slightly. 
We show that the truncated (on $r$th term) Singular Value Decomposition of ${\rm O}_{p}$ approximates well ${\rm O}_{p}$ in $||\cdot||_{2,\infty}$-norm if, basically: a) right singular vectors of ${\rm O}_{p}$ are uniformly bounded, b) the remaining part $\sum_{i=r+1}^\infty \sigma_i$ is small, where $\sigma_1\geq \sigma_2\geq \cdots$ are ordered singular values of ${\rm O}_{p}$. 
The second norm that we study is the Hilbert-Schmidt norm. Even when $||{\rm O}_{p_1}-{\rm O}_{p_2}||_{\rm HS}$ is small, there is a possibility that corresponding MMD distances between probability density functions $u,v$ will differ sharply if $||u||_{L_2}$, $||v||_{L_2}$ blow up. In this case, we only need the remaining part $\sum_{i=r+1}^\infty \sigma^2_i$ to be small in order for the truncated SVD to approximate ${\rm O}_{p}$.

Thus, the number $r^\ast$ that indicates how many moments we need to compare to approximate the MMD distance, depends on spectral properties of the operator ${\rm O}_{p}$ that defines the kernel $K$ (through $p({\mathbf x}, D)$), rather than on the spectrum of ${\rm O}_K$. This sitution is different from a set of well-known results in which spectral properties of ${\rm O}_K$ are connected with the generalization capability of kernel-based algorithms~\cite{Rademacher,Mendelson03onthe,Kloft,Rostamizadeh,Koltchinskii,Mehryar}.


Note that pseudo-differential operators have already been used in some data science tasks~\cite{potter2021parameterized,FELIUFABA2020109309,Bubba}.
\section{Preliminaries}
Throughout this paper we use standard terminology and notation from functional analysis. 
For details one can address the textbook on the theory of distributions~\cite{friedlander1998introduction}. We assume $0\in {\mathbb N}$. For $\alpha = (\alpha_1, \cdots, \alpha_n)\in {\mathbb N}^n$, $|\alpha|$  denotes $\sum_{i=1}^n\alpha_i$, $D^\alpha$ denotes $\frac{\partial^{|\alpha|}}{\partial x^{\alpha_1}_1\cdots \partial x^{\alpha_n}_n}$, ${\mathbf x}^\alpha$ denotes $x^{\alpha_1}_1\cdots x^{\alpha_n}_n$. The Schwartz space, denoted by $\mathcal{S}({\mathbb R}^n)$, is a space of infinitely differentiable functions $f: {\mathbb R}^n\rightarrow {\mathbb C}$ such that $\forall \alpha, \beta \in{\mathbb N}^n, \sup_{{\mathbf x}\in\mathbb{R}^n} $ $|{\mathbf x}^\alpha D^\beta f({\mathbf x}) |<\infty $.
Its dual space is denoted by $\mathcal{S'}({\mathbb R}^n)$. 
The Fourier and inverse Fourier transforms are denoted by $\mathcal{F}, \mathcal{F}^{-1}: \mathcal{S'}({\mathbb R}^n)\rightarrow \mathcal{S'}({\mathbb R}^n)$. For brevity, we denote $\mathcal{F}[f]$ by $\hat f$. A set of continuous functions on ${\mathbb R}^n$ is denoted by $C ({\mathbb R}^n)$, a set of infinitely differentiable functions on ${\mathbb R}^n$ is denoted by $C^\infty ({\mathbb R}^n)$. 
\subsection{Maximum mean discrepancy (MMD)}
Let $K: {\mathbb R}^n\times {\mathbb R}^n\rightarrow {\mathbb R}$ be a symmetric, positive definite kernel on a set ${\mathbb R}^n$ and ${\mathcal H}_K$ be the reproducing kernel Hilbert space that corresponds to $K$~\cite{MMD}, equipped with the norm $||\phi||_{{\mathcal H}_K} = \sqrt{\langle \phi, \phi \rangle_{{\mathcal H}_K}}$. Thus, for all $x\in {\mathbb R}^n$, $K_x = K(x, \cdot)$ is a representation of $x$ in ${\mathcal H}_K$.
It is a well-known fact that the MMD distance between probability density functions $p$ and $q$ on ${\mathbb R}^n$ can be rewritten in the following way:
\begin{equation}
\textsc{MMD}_K (p, q) =|| {\mathbb E}_{X \sim p}[ K_X ] - {\mathbb E}_{Y \sim q}[ K_Y ] ||_{{\mathcal H}_K}
\end{equation}
where we assume that ${\mathbb E}_{X \sim p}[ K_X ] = \int_{{\mathbb R}^n} K(x, y) p(y)dy$ is defined and is an element of ${\mathcal H}_K$.


\subsection{Pseudo-differential operators} 
The theory of pseudo-differential operators is a developed branch of mathematics that can be found in many textbooks, e.g. in~\cite{Taylor2017}. The basic definition is as follows (first given by Hörmander~\cite{hormander1966pseudo} and Kohn-Nirenberg~\cite{pseudodifferential}): a class $S^{m}_{1,0}$ is defined as a set of functions $F\in C^{\infty}({\mathbb R}^n\times {\mathbb R}^n)$ such that for any compact $\Omega\subseteq {\mathbb R}^n$ and any $\alpha, \beta\in {\mathbb N}^n$ there is $C_{\Omega,\alpha,\beta}>0$ such that
$$
|D^{{\mathbf x}}_{\beta}D^{{\mathbf y}}_{\alpha}F({\mathbf x},{\mathbf y})|
\leq C_{\Omega,\alpha,\beta} (1+|{\mathbf y}|)^{m-|\alpha|}$$
whenever ${\mathbf x}\in \Omega$. 

Any $F\in S^{m}_{1,0}$ defines a continuous operator $F({\mathbf x}, D): {\mathcal S}({\mathbb R}^n) \to {\mathcal S}({\mathbb R}^n)$, by the rule $$F({\mathbf x}, D) u({\mathbf x}) = \int_{{\mathbb R}^n}  F({\mathbf x},{\mathbf y})\hat{u}({\mathbf y})e^{{\rm i}{\mathbf x}^T{\mathbf y}}d {\mathbf y}.$$ We call $F({\mathbf x}, D)$ {\em the pseudo-differential operator} (PDO) with symbol $F$~\cite{Taylor2017}. 
For example, $F({\mathbf x},{\mathbf y}) = \sum_{\alpha: |\alpha|\leq m} f_\alpha({\mathbf x}) {\mathbf y}^{\alpha}$, then $F({\mathbf x}, D)$ maps $u$ to $\sum_{\alpha: |\alpha|\leq m} f_\alpha({\mathbf x}) (-{\rm i})^{|\alpha|}D^{\alpha} u({\mathbf x})$.  

If additionally, for any compact $\Omega\subseteq {\mathbb R}^n$ there is $C>0$ and $R>0$ such that $|F({\mathbf x}, {\mathbf y})|\geq C(1+|{\mathbf y}|)^{m}$ whenever ${\mathbf x}\in \Omega$ and $|{\mathbf y}|>R$, then $F({\mathbf x}, D)$ is called an elliptic operator of order $m$.

\section{PDO-based kernels}

Throughout the paper we will study a special type of Mercer kernels defined below. 
\begin{definition}
Let $F({\mathbf x}, D)$ be a PDO and let a function $K: {\mathbb R}^n\times {\mathbb R}^n\to {\mathbb R}$ satisfy
$$\langle F({\mathbf x}, D)[u], F({\mathbf x}, D)[v]\rangle_{L_2} \propto \int_{{\mathbb R}^n\times {\mathbb R}^n}\hat{u}({\mathbf x}) K({\mathbf x},{\mathbf y})\hat{v}({\mathbf y}) d {\mathbf x} d {\mathbf y}$$ 
for $u,v\in {\mathcal S}({\mathbb R}^n)$. Then, $K$ is called the PDO-based kernel and is denoted by $K^F$.
In other words, the integral operator ${\rm O}_{K^F}$ is proportional to the composition $\mathcal{F}\circ F({\mathbf x}, D)^\dag\circ F({\mathbf x}, D) \circ \mathcal{F}^{-1}$. 
\end{definition}

\begin{example} Any kernel $K({\mathbf x},{\mathbf y}) = k({\mathbf x}-{\mathbf y})$ for which $\gamma = \mathcal{F}^{-1}[k]$ satisfies $\gamma({\mathbf x})\geq 0$ and $\partial_{\alpha} \sqrt{\gamma({\mathbf x})}\in C({\mathbb R}^n)$, is defined by the PDO of the symbol $p({\mathbf x},{\mathbf y}) = \sqrt{\gamma({\mathbf x})}$. This case includes the Gaussian kernel, the Laplace kernel, the rational quadratic kernel, the Matérn kernel and many other popular kernels. 
\end{example}

\subsection{Universality of PDO-based kernels}
In this section we show that PDO-based kernels are in a certain sense universal, i.e. a continuous kernel on a compact set can be approximated uniformly with any given accuracy by some PDO-based kernel.

Let $K$ be a continuous Mercer kernel on ${\mathbb R}^n$. Thus, its restriction $K|_{\Omega}: \Omega\times \Omega\to {\mathbb R}$ on any compact $\Omega\subseteq {\mathbb R}^n$ is also a continuous Mercer kernel (and also, a Hilbert-Schmidt kernel). The following facts can be found in~\cite{Steinwart}: a) RKHS $\mathcal{H}_{K|_{\Omega}}$ is a separable space of continuous functions, b) there exists a countable basis of orthonormal functions in $\mathcal{H}_{K|_{\Omega}}$, $e_1, \cdots, e_i, ...$, such that ${\rm O}_{K|_{\Omega}} e_i = \lambda_i e_i$, c) the corresponding feature map $\Phi: {\mathbb R}^n\to \mathcal{H}_{K|_{\Omega}}$ is continuous, d) finally, Mercer's theorem gives us that
$$
K({\mathbf x},{\mathbf y}) = \langle \Phi({\mathbf x}),\Phi ({\mathbf y})\rangle_{\mathcal{H}_K} = \sum_{i=1}^\infty f^\ast_{i} ({\mathbf x}) f_{i} ({\mathbf y}) 
$$
where $f_{i} ({\mathbf x}) = \langle e_i, \Phi({\mathbf x})\rangle_{\mathcal{H}_K}$ and the convergence is absolute and uniform on $\Omega$.  

Therefore, let us study the kernel
\begin{equation}\label{form}
K({\mathbf x},{\mathbf y}) = \sum_{i=1}^N e^{-\epsilon^2 ||{\mathbf x}-{\mathbf y}||^2}f^\ast_{i} ({\mathbf x}) f_{i} ({\mathbf y})
\end{equation}
where $\epsilon>0$, $f_{i}({\mathbf x}) = \sum_{\alpha: |\alpha|\leq m} c^i_{\alpha}{\mathbf x}^\alpha, i\in [N]$ are polynomials such that $c^i_{\alpha}\ne 0$ if $|\alpha|=m$. Obviously, finite sums of the form~\eqref{form} can approximate uniformly continuous kernels on compact sets.
\begin{proposition} The kernel~\eqref{form} is PDO-based.
\end{proposition}
\begin{proof} The following bilinear forms are proportional:
\begin{equation*}
\begin{split}
\langle \hat{u}, {\rm O}_K \hat{v}\rangle = \sum_{i=1}^N \langle f_{i} \hat{u}, {\rm O}_{e^{-\epsilon^2 ||{\mathbf x}-{\mathbf y}||^2}} f_{i}\hat{v} \rangle \propto \\
\sum_{i=1}^N \langle f_{i}(D) u, e^{- ||{\mathbf x}||^2/4\epsilon^2} f_{i}(D)v  \rangle = \langle  u, L v  \rangle
\end{split}
\end{equation*}
where $L = \sum_{i=1}^N f_{i}(D)^\dag e^{- ||{\mathbf x}||^2/4\epsilon^2} f_{i}(D)$. 

Since $\langle  u, L u  \rangle\geq 0$, then $L$ is a positive semidefinite self-adjoint operator. 
Additionally, $L$ is an elliptic operator of order $2m$, because there is 
\begin{equation}\label{elliptic}
p({\mathbf x}, {\mathbf y}) = \sum\limits_{\alpha, \beta: |\alpha|\leq 2m, |\beta|\leq 2m} c_{\alpha, \beta} {\mathbf x}^{\alpha} e^{- ||{\mathbf x}||^2/4\epsilon^2} {\mathbf y}^{\beta}\in S^{2m}_{1,0}
\end{equation}
such that $L = p({\mathbf x}, D)$. The ellipticity property is satisfied because the principal part of $p({\mathbf x}, {\mathbf y})$ is $$\sum_{i=1}^N\sum\limits_{\alpha: |\alpha|=m} |c^i_{\alpha}|^2 e^{- ||{\mathbf x}||^2/4\epsilon^2} {\mathbf y}^{2\alpha}.$$

According to Seeley (see~\cite{Schrohe1986}) the square root for such an operator is also elliptic self-adjoint pseudo-differential operator of order $m$. In other words, there exists a function $F({\mathbf x},{\mathbf y})\in S^{m}_{1,0}$
such that $F({\mathbf x},D)F({\mathbf x},D)=p({\mathbf x}, D)$. It is easy to see that $\langle F({\mathbf x}, D)[u], F({\mathbf x}, D)[v]\rangle_{L_2} \propto \int_{{\mathbb R}^n\times {\mathbb R}^n}\hat{u}({\mathbf x}) K({\mathbf x},{\mathbf y})\hat{v}({\mathbf y}) d {\mathbf x} d {\mathbf y}$. Therefore, $K$ is defined by $F$, i.e. $K=K^F$.
\end{proof}




\section{GAN-style interpretation of MMD}
For a probability density function $u\in {\mathcal S}({\mathbb R}^n)$, $\chi_u$ denotes its characteristic function, i.e. $\chi_u({\mathbf x}) = {\mathbb E}_{\boldsymbol{\xi}\sim u}e^{{\rm i}\boldsymbol{\xi}^T{\mathbf x}}= \mathcal{F}^{-1}[u]({\mathbf x})$. For $F\in S^m_{1,0}$, $F(D, {\mathbf x})$ denotes $\mathcal{F}\circ F({\mathbf x}, D)^\dag\circ \mathcal{F}^{-1}$, a continuous operator from ${\mathcal S}({\mathbb R}^n)$ to ${\mathcal S}({\mathbb R}^n)$.
\begin{proposition} Let $u, v\in {\mathcal S}({\mathbb R}^n)$ be two probability density functions and $F({\mathbf x}, D)$ be a PDO. Then, the MMD distance defined by $K^F$ satisfies:
\begin{equation}\label{derivation}
\begin{split}
{\rm MMD}_{K^F}(u,v) \propto \\
\sup_{||f||_{L_2({\mathbb R}^n)}\leq 1}  {\mathbb E}_{X\sim u}F(D, {\mathbf x}) f (X) - {\mathbb E}_{Y\sim v}F(D, {\mathbf x}) f (Y)
\end{split}
\end{equation}
\end{proposition}
\ifTR
\begin{proof}
The MMD distance defined by $K^F$ can be expressed:
\begin{equation}\label{derivation}
\begin{split}
{\rm MMD}_{K^F}(u,v) \propto ||F({\mathbf x}, D)[\chi_u]-F({\mathbf x}, D)[\chi_v] ||_{L_2} = \\
\sup_{f\in {\mathcal S}({\mathbb R}^n): ||f||_{L_2({\mathbb R}^n)}\leq 1} \langle f, F({\mathbf x}, D)[\chi_u]-F({\mathbf x}, D)[\chi_v]\rangle_{L_2} = \\
\sup_{f\in {\mathcal S}({\mathbb R}^n): ||f||_{L_2({\mathbb R}^n)}\leq 1} \langle F({\mathbf x}, D)^\dag f, \chi_u-\chi_v\rangle_{L_2} \propto \\
\sup_{f\in {\mathcal S}({\mathbb R}^n): ||f||_{L_2({\mathbb R}^n)}\leq 1} \langle \mathcal{F}\big\{F({\mathbf x}, D)^\dag f\big\}, u-v\rangle_{L_2} = \\
\sup_{f\in {\mathcal S}({\mathbb R}^n): ||f||_{L_2({\mathbb R}^n)}\leq 1} \langle F(D, {\mathbf x}) \hat{f}, u-v\rangle_{L_2} \propto  \\
\sup_{f\in {\mathcal S}({\mathbb R}^n): ||f||_{L_2({\mathbb R}^n)}\leq 1} \langle F(D, {\mathbf x}) f, u-v\rangle_{L_2} = \\
\sup_{||f||_{L_2({\mathbb R}^n)}\leq 1}  {\mathbb E}_{X\sim u}F(D, {\mathbf x}) f (X) - {\mathbb E}_{Y\sim v}F(D, {\mathbf x}) f (Y)
\end{split}
\end{equation}
\end{proof}
\else
\fi


From the latter proposition the following is straightforward.
\begin{proposition} Let $F=F_1+F_2$, then
$${\rm MMD}_{K^{F}}(u,v)\leq {\rm MMD}_{K^{F_1}}(u,v)+{\rm MMD}_{K^{F_2}}(u,v)$$
\end{proposition}

The last representation of MMD distance allows to put the MMD minimization into the minimax framework.
Let us now consider a pseudo-differential operator with a symbol $F$ such that the range of ${\rm O}_F$ is finite-dimensional. 
\subsection{A finite sum case}\label{finite-case}
In applications, a finite sum case is important:
\begin{equation}\label{finite}
F({\mathbf x},{\mathbf y}) = \sum_{i=1}^l f_i({\mathbf x}) g_i({\mathbf y})
\end{equation}
where $\widehat{f_i}\in L_{1}({\mathbb R}^n)$. The latter implies that $f_i$ is bounded. 
\begin{proposition}\label{simple} If $F$ satisfies~\eqref{finite}, then it can be represented as $F({\mathbf x},{\mathbf y}) = \sum_{i=1}^L h_i({\mathbf x}) t_i({\mathbf y})$ where $L\leq 4l$, $\widehat{h_i}$ is a pdf, i.e. $\widehat{h_i} ({\mathbf x})\geq 0, \int_{{\mathbb R}^n}\widehat{h_i} ({\mathbf x})d{\mathbf x} = 1$.
\end{proposition}
\ifTR
\begin{proof} By substituting each term $f_i({\mathbf x}) g_i({\mathbf y}) = (\frac{f_i({\mathbf x})+f_i(-{\mathbf x})}{2}+\frac{f_i({\mathbf x})-f_i(-{\mathbf x})}{2}) g_i({\mathbf y})$ with two terms $\frac{f_i({\mathbf x})+f_i(-{\mathbf x})}{2} g_i({\mathbf y}) +\frac{f_i({\mathbf x})-f_i(-{\mathbf x})}{2{\rm i}} {\rm i} g_i({\mathbf y})$ we can turn a sum with $l$ terms into a sum with $2l$ terms in such a way that in all resulting terms $h_i({\mathbf x}) t_i({\mathbf y})$, $\widehat{h}_i({\mathbf x})$ is a real valued function. This operation costs us the doubling of the number of terms. 

Let us now consider the term $f_i({\mathbf x}) g_i({\mathbf y})$ such that $\widehat{f}_i({\mathbf x})$ is a real valued function. 

Suppose $\mathcal{F}[f_i] = f'_i$ and both inverse images $(f'_i)^{-1}[{\mathbb R}_+]$ and $(f'_i)^{-1}[{\mathbb R}_-]$ have nonzero Lebegue measure.
Note that $f'_i\in L_{1}({\mathbb R}^n)$, therefore $f'_i({\mathbf x}) = \alpha p^+_i({\mathbf x})-\beta p^-_i({\mathbf x})$, where $p^+_i({\mathbf x}) = \max\{0,f'_i({\mathbf x})\}/\alpha$, $p^-_i({\mathbf x}) = -\min\{0,f'_i({\mathbf x})\}/\beta$ and $\alpha = ||\max\{0,f'_i({\mathbf x})\}||_{L_1}$, $\beta = ||\min\{0,f'_i({\mathbf x})\}||_{L_1}$. Thus, we can substitute $f_i({\mathbf x}) g_i({\mathbf y})$ with $\tilde{f}^1_i({\mathbf x}) \tilde{g}^1_i({\mathbf y})+\tilde{f}^2_i({\mathbf x}) \tilde{g}^2_i({\mathbf y})$ where $\tilde{f}^1_i = \mathcal{F}^{-1}[p^+_i]$,  $g^1_i({\mathbf y}) = \alpha g_i({\mathbf y})$, $\tilde{f}^2_i = \mathcal{F}^{-1}[p^-_i]$ and  $g^2_i({\mathbf y}) = -\beta g_i({\mathbf y})$.

If $f'_i({\mathbf x})\geq 0$ a.s., then we can substitute $f_i({\mathbf x}) g_i({\mathbf y})$ with $\tilde{f}_i({\mathbf x}) \tilde{g}_i({\mathbf y})$ where $\tilde{f}_i({\mathbf x}) = \frac{f_i({\mathbf x})}{||f'_i||_{L_1}}$ and  $\tilde{g}_i({\mathbf y}) = ||f'_i||_{L_1} g_i({\mathbf y})$. Analogously the case $f_i({\mathbf x})\leq 0$ is treated. 

Thus, any such term can be substituted with 2 terms of the form $h_i({\mathbf x}) t_i({\mathbf y})$ where $\widehat{h_i}$ is a pdf. Both operations cost at most the doubling of the number of terms. If we apply them sequentially, then we obtain at most $4l$ terms.
\end{proof}
\else
\fi

Due to proposition~\ref{simple}, let us additionally assume in~\eqref{finite} that $\widehat{f_i} = p_i$ and $p_i ({\mathbf x})\geq 0, \int_{{\mathbb R}^n}p_i ({\mathbf x})d{\mathbf x} = 1$.
Then, we obtain:
\begin{equation*}
\begin{split}
f({\mathbf x}) \stackrel{F(D, {\mathbf x})}{\to} \sum_{i=1}^l g^\ast_i \mathcal{F}[f^\ast_i\mathcal{F}^{-1}[f]] \propto \sum_{i=1}^l g^\ast_i({\mathbf x}) (q_i\ast f)({\mathbf x})
\end{split}
\end{equation*}
where $q_i ({\mathbf x}) = p_i(-{\mathbf x})=\widehat{f^\ast_i}$.
Thus, if $u$ is a smooth probability density function, then:
\begin{equation*}
\begin{split}
{\mathbb E}_{{\mathbf x}\sim u}F(D, {\mathbf x}) f ({\mathbf x}) \propto {\mathbb E}_{{\mathbf x}\sim u} \sum_{i=1}^l g^\ast_i({\mathbf x}) (q_i\ast f)({\mathbf x})  = \\
\sum_{i=1}^l \int_{{\mathbb R}^n\times {\mathbb R}^n} u({\mathbf x}) g^\ast_i({\mathbf x}) q_i({\mathbf y}) f({\mathbf x}-{\mathbf y}) d{\mathbf x} d {\mathbf y} = \\
\sum_{i=1}^l {\mathbb E}_{{\mathbf x}\sim u, {\mathbf y}\sim q_i} g^\ast_i({\mathbf x})  f({\mathbf x}-{\mathbf y})
\end{split}
\end{equation*}
Now, using~\eqref{derivation} we finally obtain a characterization of MMD as supremum:
\begin{equation}\label{gan-pure}
\begin{split}
{\rm MMD}_{K^F}(u,v) \propto \\
\sup_{f: ||f||_{L_2({\mathbb R}^n)}\leq 1}\sum_{i=1}^l {\mathbb E}_{{\mathbf x}\sim u, {\mathbf y}\sim q_i} g^\ast_i({\mathbf x})  f({\mathbf x}-{\mathbf y}) - \\
{\mathbb E}_{{\mathbf x}\sim v, {\mathbf y}\sim q_i} g^\ast_i({\mathbf x})  f({\mathbf x}-{\mathbf y})
\end{split}
\end{equation}
Note that if $l=1$ and $g_1({\mathbf x})=1$, ${\rm MMD}_{K^F}(u,v)$ is proportional to:
$$
\sup\limits_{f: ||f||_{L_2({\mathbb R}^n)}\leq 1} \mathop{\mathbb E}\limits_{{\mathbf x}\sim u, \boldsymbol{\epsilon}\sim p_1} f({\mathbf x}+\boldsymbol{\epsilon}) - \mathop{\mathbb E}\limits_{{\mathbf x}\sim v, \boldsymbol{\epsilon}\sim p_1}  f({\mathbf x}+\boldsymbol{\epsilon})
$$
where $q_1({\mathbf x}) = p_1(-{\mathbf x})$.
In this case we have $K({\mathbf x}, {\mathbf y}) = \widehat{f^2_1}({\mathbf x}-{\mathbf y})$ (the translation invariant kernel case~\cite{ghiasi-shirazi10a}).
This case captures the Gaussian kernel, the Poisson kernel, the Abel kernel, and many other kernels. 

From the last expression the GAN-style interpretation of generative networks based on MMD distance directly follows. In the minimax task~\eqref{main-problem}, which is equivalent to $\min\limits_{\theta}\sup\limits_{f: ||f||_{L_2({\mathbb R}^n)}\leq 1} {\mathbb E}_{{\mathbf x}\sim \mu_\theta, \boldsymbol{\epsilon}\sim p_1} f({\mathbf x}+\boldsymbol{\epsilon}) - {\mathbb E}_{{\mathbf x}\sim \mu_{\rm emp}, \boldsymbol{\epsilon}\sim p_1}  f({\mathbf x}+\boldsymbol{\epsilon})$, the generator always adds noise, distributed according to $p_1({\mathbf x})$, to vectors generated from empirical and model distributions.
The critic optimizes over functions from unit ball in $L_2({\mathbb R}^n)$ and tries to increase the $f$'s value in regions for which probabilities of $\mu_{\rm emp}$ and $\mu_{\theta}$ differ.

Note that the GAN-style interpretation of~\eqref{gan-pure} is also straightforward. Let us introduce:
\begin{equation}\label{local-moment}
\begin{split}
m_i ({\mathbf t}, u) = {\mathbb E}_{{\mathbf x}\sim u, \boldsymbol{\epsilon}\sim p_i} [g^\ast_i({\mathbf x}) |  {\mathbf x} + \boldsymbol{\epsilon} = {\mathbf t}]
\end{split}
\end{equation}
for $q_i({\mathbf x}) = p_i(-{\mathbf x})$ and $i\in [l]$. The expression~\eqref{local-moment} is the expectation of the feature $g^\ast_i({\mathbf x})$ when ${\mathbf x}$ is sampled according to distribution $u$, under condition that the sum of 
${\mathbf x}$ and the noise vector $\boldsymbol{\epsilon}$ equals ${\mathbf t}$. In other words, $m_i ({\mathbf t}, u)$ is a $g^\ast_i$-moment of $u$ in the neighbourhood of ${\mathbf t}$.

Thus,
\begin{equation*}
\begin{split}
{\rm MMD}_{K^F}(u,v) \propto \sup_{f: ||f||_{L_2({\mathbb R}^n)}\leq 1}\sum_{i=1}^l\\
 {\mathbb E}_{{\mathbf x}\sim u, \boldsymbol{\epsilon}\sim p_i} f({\mathbf x} + \boldsymbol{\epsilon}) \big(m_i ({\mathbf x} + \boldsymbol{\epsilon}, u)-m_i ({\mathbf x} + \boldsymbol{\epsilon}, v)\big)
\end{split}
\end{equation*}
From the latter formula we can make the following conclusions:
\begin{itemize}
\item The critic tries to increase $f$ in regions for which ``local'' expectations of $g^\ast_i({\mathbf x})$ for empirical and model distributions, i.e. $m_i ({\mathbf x} + \boldsymbol{\epsilon}, u)-m_i ({\mathbf x} + \boldsymbol{\epsilon}, v)$, differ sharply. Thus, the MMD distance measures the difference between distributions of the feature vector $G({\mathbf x}) = [g^\ast_i({\mathbf x})]_{i=1}^l$  for ${\mathbf x}\sim u$ and ${\mathbf x}\sim v$ around each point ${\mathbf t}$.
\item Different features affect the critic's objective simultaneously. That is why it is important to generate noise for different features, $g^\ast_i$ and $g^\ast_j$, according to different distributions, $p_i$ and $p_j$. Otherwise, if $p_i = p_j$, then the MMD distance will compare the distributions of a single feature, $g^\ast_i + g^\ast_j$. For example, suppose $p_i$ is a pdf of the Gaussian noise with the bandwidth $\sigma_i$. In the design of MMD's kernel, the bandwidth $\sigma_i$ rather than being the same for all features, better be adapted to the feature $g^\ast_i$. E.g. for the kurtosis of a distribution (4th moment), it is natural to have a larger bandwidth than for the expected value.
\end{itemize}

Let us denote ${\rm Feat}_F = \{\lambda_1g_1+...+\lambda_l g_l| \lambda_i\in {\mathbb C}\}$, i.e. the span of those functions.
Thus, it becomes clear that ${\rm Feat}_F$ serves as a space of features whose distribution we compare. The dimension of  ${\rm Feat}_F $ can be naturally called a dimension of the kernel $K^F$. It is easy to see that this dimension equals ${\rm dim}\, \mathcal{R} [O^\dag_{F}]$.

\section{A natural norm on PDOs defining kernels}
For an operator $O: \mathcal{S}({\mathbb R}^n)\to \mathcal{S}({\mathbb R}^n)$ let us define
$$||O||_{2,\infty} = \sup\limits_{f\in {\mathcal S}({\mathbb R}^n): ||f||_{L_2({\mathbb R}^n)}\leq 1}|| Of({\mathbf x})||_{L_\infty({\mathbb R}^n)} .$$

For any function $F: {\mathbb R}^n \times {\mathbb R}^n\to {\mathbb R}$ such that $A = \sup_{\mathbf x} \int_{{\mathbb R}^n} |F({\mathbf x},{\mathbf y})|^2 d{\mathbf y} < \infty$, we have $||{\rm O}_{F}||_{2,\infty} = \sqrt{A}$. In other words, the norm $||\cdot||_{2,\infty}$ can be easily calculated for integral operators. The following two proposition show that the calculation of $||F({\mathbf x}, D)||_{2,\infty}$ is also simple.

\begin{proposition}\label{diag} If $F\in S^{m}_{1,0}, ||{\rm O}_{F}||_{2,\infty}<\infty$, then diagonal elements of Schwartz kernels of $F({\mathbf x}, D) F({\mathbf x}, D)^\dag$ and ${\rm O}_{F}{\rm O}^\dag_{F}$ are the same, i.e. there are smooth functions $M,N$ such that $F({\mathbf x}, D) F({\mathbf x}, D)^\dag = {\rm O}_{M}$, ${\rm O}_{F}{\rm O}^\dag_{F} = {\rm O}_{N}$ and $M(\boldsymbol{\xi}, \boldsymbol{\xi}) \propto N(\boldsymbol{\xi}, \boldsymbol{\xi})$.
\end{proposition}
\begin{proof} The Schwartz kernel of $F({\mathbf x}, D)$ is $K_1(\boldsymbol{\xi}', {\mathbf x}) = \int e^{{\rm i}\boldsymbol{\xi}'^T{\mathbf y}'} F(\boldsymbol{\xi}', {\mathbf y}') e^{-{\rm i}{\mathbf x}^T{\mathbf y}'} d {\mathbf y}' $. Therefore, the Schwartz kernel of $F({\mathbf x}, D)^\dag$ is $K_2( {\mathbf x}, \boldsymbol{\xi}) = K_1(\boldsymbol{\xi},{\mathbf x})^\ast = \int e^{{\rm i}{\mathbf x}^T{\mathbf y}} F(\boldsymbol{\xi}, {\mathbf y})^\ast e^{-{\rm i}\boldsymbol{\xi}^T{\mathbf y}} d {\mathbf y}$.
Thus,
\begin{equation*}
\begin{split}
\langle  u,  F({\mathbf x}, D)F({\mathbf x}, D)^\dag v\rangle \propto \\
\int u(\boldsymbol{\xi}')^\ast e^{{\rm i}\boldsymbol{\xi}'^T{\mathbf y}'} F(\boldsymbol{\xi}', {\mathbf y}') e^{-{\rm i}{\mathbf x}^T{\mathbf y}'} e^{{\rm i}{\mathbf x}^T{\mathbf y}} F(\boldsymbol{\xi}, {\mathbf y})^\ast e^{-{\rm i}\boldsymbol{\xi}^T{\mathbf y}} v(\boldsymbol{\xi}) \\ d \boldsymbol{\xi} d {\mathbf y} d {\mathbf x}  d {\mathbf y}' d \boldsymbol{\xi}'
\end{split}
\end{equation*}
Thus, the Schwartz kernel of $F({\mathbf x}, D)F({\mathbf x}, D)^\dag$ is
\begin{equation*}
\begin{split}
M(\boldsymbol{\xi}',\boldsymbol{\xi}) = \\
\int e^{{\rm i}\boldsymbol{\xi}'^T{\mathbf y}'} F(\boldsymbol{\xi}', {\mathbf y}') e^{-{\rm i}{\mathbf x}^T{\mathbf y}'} e^{{\rm i}{\mathbf x}^T{\mathbf y}} F(\boldsymbol{\xi}, {\mathbf y})^\ast e^{-{\rm i}\boldsymbol{\xi}^T{\mathbf y}}  d {\mathbf y} d {\mathbf y}'   d {\mathbf x} 
\end{split}
\end{equation*}
The diagonal element equals:
\begin{equation*}
\begin{split}
M(\boldsymbol{\xi},\boldsymbol{\xi}) = \\
\int e^{{\rm i}\boldsymbol{\xi}^T{\mathbf y}'} F(\boldsymbol{\xi}, {\mathbf y}') e^{-{\rm i}{\mathbf x}^T{\mathbf y}'} e^{{\rm i}{\mathbf x}^T{\mathbf y}} F(\boldsymbol{\xi}, {\mathbf y})^\ast e^{-{\rm i}\boldsymbol{\xi}^T{\mathbf y}}  d {\mathbf y} d {\mathbf y}'   d {\mathbf x}  = \\
\int e^{-{\rm i}(\boldsymbol{\xi}-{\mathbf x})^T({\mathbf y}-{\mathbf y}')} F(\boldsymbol{\xi}, {\mathbf y}') F(\boldsymbol{\xi}, {\mathbf y})^\ast   d {\mathbf y} d {\mathbf y}' d {\mathbf x}   = \\
{\rm after\,\,we\,\,set\,\,}G_{\mathbf b}({\mathbf a}) = F({\mathbf b},-{\mathbf a}), F_{\mathbf b}({\mathbf a}) = F({\mathbf b}, {\mathbf a})^\ast \\
\int e^{-{\rm i}(\boldsymbol{\xi}-{\mathbf x})^T\boldsymbol{\lambda}} G_{\boldsymbol{\xi}}(\boldsymbol{\lambda}-{\mathbf y}) F_{\boldsymbol{\xi}}({\mathbf y})  d {\mathbf y} d \boldsymbol{\lambda}  d {\mathbf x}  = \\
\int e^{-{\rm i}(\boldsymbol{\xi}-{\mathbf x})^T\boldsymbol{\lambda}} \big\{G_{\boldsymbol{\xi}}\ast F_{\boldsymbol{\xi}}\big\} (\boldsymbol{\lambda})  d \boldsymbol{\lambda} d {\mathbf x}   \propto \\
\int  \mathcal{F} \big[\big\{G_{\boldsymbol{\xi}}\ast F_{\boldsymbol{\xi}}\big\} \big](\boldsymbol{\xi}-{\mathbf x}) d {\mathbf x} =  
\int  \widehat{G}_{\boldsymbol{\xi}}(\boldsymbol{\xi}-{\mathbf x}) \widehat{F_{\boldsymbol{\xi}}}(\boldsymbol{\xi}-{\mathbf x}) d {\mathbf x} = \\ 
\langle \widehat{G}_{\boldsymbol{\xi}}^\ast, \widehat{F_{\boldsymbol{\xi}}} \rangle_{L_2} = \langle G_{\boldsymbol{\xi}} (-{\mathbf x})^\ast, F_{\boldsymbol{\xi}}({\mathbf x}) \rangle_{L_2} = \int F(\boldsymbol{\xi}, {\mathbf x}) F(\boldsymbol{\xi}, {\mathbf x})^\ast d{\mathbf x}
\end{split}
\end{equation*}
The latter equals the diagonal element of ${\rm O}_{F}{\rm O}^\dag_{F}$.
\end{proof}

\begin{proposition} If $F\in S^{m}_{1,0}$, then
\begin{equation}
\begin{split}
||F({\mathbf x}, D)||_{2,\infty} \propto ||{\rm O}_{F}||_{2,\infty} = \sqrt{\sup_{\mathbf x} \int_{{\mathbb R}^n} |F({\mathbf x},{\mathbf y})|^2 d{\mathbf y}}
\end{split}
\end{equation}
\end{proposition}
\begin{proof} The norm of the integral operator ${\rm O}_{F}$ can be expressed as:
\begin{equation*}
\begin{split}
||{\rm O}_{F}||_{2,\infty} = \sup_{\mathbf x} \sup_{u: ||u||_{L_2({\mathbb R}^n)}\leq 1}\int F({\mathbf x},{\mathbf y}) u({\mathbf y}) d{\mathbf y} = \\
\sup_{\mathbf x} ||F({\mathbf x},\cdot)||_{L_2} = \sup_{\mathbf x}\sqrt{ \int_{{\mathbb R}^n} |F({\mathbf x},{\mathbf y})|^2 d{\mathbf y}} = \\
\sqrt{\sup_{\mathbf x} \int_{{\mathbb R}^n} F({\mathbf x},{\mathbf y})^\ast F({\mathbf x},{\mathbf y}) d{\mathbf y}} = \sqrt{\sup_{\mathbf x} N_F({\mathbf x},{\mathbf x})}
\end{split}
\end{equation*}
where $N_F$ is the Schwartz kernel of ${\rm O}_{F}{\rm O}^\dag_{F}$. Let ${\rm O}_{F'} = F({\mathbf x}, D)$. Due to proposition~\ref{diag}, diagonal elements of ${\rm O}_{F'}{\rm O}^\dag_{F'}$  and ${\rm O}_{F}{\rm O}^\dag_{F}$ are the same, i.e $N_{F'}({\mathbf x},{\mathbf x})=N_{F}({\mathbf x},{\mathbf x})$. Therefore, supremum of those diagonal elements, $\sup_{\mathbf x} N_F({\mathbf x},{\mathbf x})$ and $\sup_{\mathbf x} N_{F'}({\mathbf x},{\mathbf x})$, are also the same. Thus,  $||F({\mathbf x}, D)||_{2,\infty} = ||{\rm O}_{F'}||_{2,\infty} = ||{\rm O}_{F}||_{2,\infty}$.
\end{proof}
\begin{proposition}\label{norm} If $F_1\in S^{m_1}_{1,0}$ and $F_2\in S^{m_2}_{1,0}$, then
\begin{equation}
\begin{split}
|{\rm MMD}_{K^{F_1}}(u,v)-{\rm MMD}_{K^{F_1}}(u,v)|\leq \\
c||F_1(D, {\mathbf x})-F_2(D, {\mathbf x}))||_{2,\infty}
\end{split}
\end{equation}
where $c$ is some constant.
\end{proposition}
\begin{proof}
Using~\eqref{derivation} we can express:
\begin{equation*}
\begin{split}
{\rm MMD}_{K^{F_1}}(u,v)-{\rm MMD}_{K^{F_2}}(u,v) = \\
\sup_{f_1\in {\mathcal S}({\mathbb R}^n): ||f_1||_{L_2({\mathbb R}^n)}\leq 1} \langle F_1(D, {\mathbf x}) f_1, u-v\rangle_{L_2} -  \\
\sup_{f_2\in {\mathcal S}({\mathbb R}^n): ||f_2||_{L_2({\mathbb R}^n)}\leq 1} \langle F_2(D, {\mathbf x}) f_2, u-v\rangle_{L_2} \leq  \\
\sup_{f\in {\mathcal S}({\mathbb R}^n): ||f||_{L_2({\mathbb R}^n)}\leq 1} \langle (F_1(D, {\mathbf x})-F_2(D, {\mathbf x})) f, u-v\rangle_{L_2} \leq \\
{\rm using\,\,Holder\,\,inequality} \\
\sup_{f\in {\mathcal S}({\mathbb R}^n): ||f||_{L_2({\mathbb R}^n)}\leq 1} ||F_1(D, {\mathbf x})-F_2(D, {\mathbf x})) f||_{L_\infty} ||u-v||_{L_1} \\
\leq 2||F_1(D, {\mathbf x})-F_2(D, {\mathbf x}))||_{2,\infty}
\end{split}
\end{equation*}
because $||u-v||_{L_1}\leq 2$ and $||F_1(D, {\mathbf x})-F_2(D, {\mathbf x})) f||_{L_\infty}\leq ||F_1(D, {\mathbf x})-F_2(D, {\mathbf x}))||_{2,\infty}$. Analogously, we bound ${\rm MMD}_{K^{F_2}}(u,v)-{\rm MMD}_{K^{F_1}}(u,v)$ and this completes the proof.
\end{proof}
A natural norm on the PDO $F({\mathbf x}, D)$ defining the MMD distance $K^F$, is $||F(D, {\mathbf x})||_{2,\infty}$. A partial justification of this statement is the proposition~\eqref{norm} and the following proposition.
\begin{proposition}\label{remainder} Let $F\in S^{m}_{1,0}$, $f_i\in L_{2}({\mathbb R}^n)$, $g_i\in L_{\infty}({\mathbb R}^n)$ for $i=\overline{1,\infty}$, and $\sum_{i=1}^N f_i({\mathbf x}) g_i({\mathbf y})\stackrel{N\rightarrow \infty}{\to} F$ wrt $||\cdot||_{2,\infty}$ norm.  Then, 
\begin{equation*}
\begin{split}
||F(D, {\mathbf x})||_{2,\infty} \leq C\sum_{i=1}^\infty ||f_i||_{L_2} ||g_i||_{L_{\infty}}.
\end{split}
\end{equation*}
\end{proposition}
\begin{proof} Let $F_N({\mathbf x}, {\mathbf y}) = \sum_{i=1}^N f_i({\mathbf x}) g_i({\mathbf y})$, $F_N({\mathbf x}, D)$ is a continuous operator from $L_{2}({\mathbb R}^n)$ to $L_{2}({\mathbb R}^n)$ defined by $F_N({\mathbf x}, D)f({\mathbf x}) = \int_{{\mathbb R}^n} F_N({\mathbf x}, {\mathbf y})\hat{f}({\mathbf y}) d{\mathbf y}$, and $F_N(D, {\mathbf x})$ is $\mathcal{F}\circ F_N({\mathbf x}, D) \circ \mathcal{F}^{-1}$. 

For any $f: ||f||_{L_{2}({\mathbb R}^n)}\leq 1$ we have
\begin{equation*}
\begin{split}
||F_N(D, {\mathbf x})f||_{L_{\infty}}\leq \sum_{i=1}^N ||g^\ast_i \mathcal{F}[f^\ast_i\mathcal{F}^{-1}[f]]||_{L_{\infty}} \leq \\ 
\sum_{i=1}^N ||g^\ast_i||_{L_{\infty}} ||\mathcal{F}[f^\ast_i\mathcal{F}^{-1}[f]]||_{L_{\infty}}
\end{split}
\end{equation*}
Since,
\begin{equation*}
\begin{split}
|\mathcal{F}[f^\ast_i\mathcal{F}^{-1}[f]](\boldsymbol{\xi})| \propto |\int f^\ast_i({\mathbf x}) \mathcal{F}^{-1}[f]({\mathbf x})e^{-{\rm i}\boldsymbol{\xi}^T{\mathbf x}}d{\mathbf x}| \leq \\
\int |f^\ast_i({\mathbf x}) \mathcal{F}^{-1}[f]({\mathbf x})| d{\mathbf x} \leq \\
{\rm Using\,\,Cauchy-Schwarz\,\,inequality}\\
\sqrt{\int |f^\ast_i({\mathbf x})|^2  d{\mathbf x} \int |\mathcal{F}^{-1}[f]({\mathbf x})|^2  d{\mathbf x}}\leq 
C||f_i||_{L_2} ||f||_{L_2} 
\end{split}
\end{equation*}
we obtain the bound $||F_N(D, {\mathbf x})||_{2,\infty}\leq C\sum_{i=1}^N ||f_i||_{L_2} ||g_i||_{L_{\infty}}$. Finally, using $F_N\stackrel{N\rightarrow \infty}{\to} F$ we obtain the needed inequality.
\end{proof}


\section{Intrinsic dimensionality of PDO-based kernels}
In Section~\ref{finite-case} we learned that if $G\in S^m_{1,0}$ satisfies $\mathcal{R} [O^\dag_{G}]=r$, then ${\rm MMD}_{K^G}$ reflects the similarity of two distributions with respect to $r$ local moments.
Therefore, it is a natural idea to approximate arbitrary PDO-based kernel by some $K^G$ where $\mathcal{R} [O^\dag_{G}]=r$.

Let $K$ be a Mercer kernel such that $${\rm Sym}(K) = \{F\in S^m_{1,0}| m\in {\mathbb N}, K^F=K\}$$ is nonempty. Let us introduce
\begin{equation}
d_{K} (r) = \inf\limits_{(F,G)\in \Pi(K,r)} ||F(D, {\mathbf x})-G(D, {\mathbf x})||_{2,\infty}
\end{equation}
where $$\Pi(K,r) = \{(F,G)|F\in {\rm Sym}(K), G\in S^m_{1,0}, {\rm dim}\, \mathcal{R} [O^\dag_{G}]=r\}.$$
The function $d_{K} (r)$ plays the same role as {\em the retained variance} in the Principal Component Analysis.
We will study the behaviour of $d_{K} (r)$ as $r\to\infty$.

\begin{proposition} Let $K$ be a Mercer kernel, $F\in {\rm Sym}(K)$ be a Hilbert-Schmidt kernel, $\sigma_{1}\geq \sigma_{2}\geq \cdots $ be singular values of the operator ${\rm O}_{F}$ such that $\sum_{i=1}^\infty \sigma_i$ converges. Also, let $$C_F = \sup\limits_{u\ne 0, {\rm O}^\dag_{F}{\rm O}_{F}u = \lambda u}\frac{||u||_\infty}{||u||_2}<\infty.$$ Then, 
$$d_{K} (r) \leq C_F
\sum_{i=r+1}^\infty \sigma_i.$$
\end{proposition}
Note that requiring $C_F<\infty$ means that right singular vectors of ${\rm O}_{F}$ (or, eigenvectors of ${\rm O}^\dag_{F}{\rm O}_{F}$) are uniformly bounded. This condition is popular in various statements concerning Mercer kernels, though it is believed that it is hard to check. Discussions of that issue can be found in~\cite{Zhou,Regularizationkernel,Steinwart2012}. 

\begin{proof}
Using Singular Value Decomposition we obtain:
$$
\sum_{i=1}^N \sigma_{i} f_i({\mathbf x}) g_i({\mathbf y})^\ast \stackrel{N\to\infty}{\to} F({\mathbf x}, {\mathbf y}){\rm\,\,in\,\,}L_{2}({\mathbb R}^n\times {\mathbb R}^n)
$$
where $\{f_i\}_{i=1}^\infty\subseteq L_{2}({\mathbb R}^n)$ and $\{g_i\}_{i=1}^\infty\subseteq L_{2}({\mathbb R}^n)$ are systems of orthonormal vectors in $L_{2}({\mathbb R}^n)$ and $g_i$ is an eigenvector of ${\rm O}^\dag_{F}{\rm O}_{F}$. 

Since $C_F<\infty$, we can write
\begin{equation*}
\begin{split}
||g_i||_{L_\infty}\leq C_F
\end{split}
\end{equation*}
As in the proof of proposition~\ref{remainder}, let us denote $F_r ({\mathbf x}, {\mathbf y}) = \sum_{i=1}^r \sigma_{i} f_i({\mathbf x}) g_i({\mathbf y})^\ast$ and bound 
\begin{equation*}
\begin{split}
||\sum_{i=r+1}^{r+N} \sigma_{i} f_i(D) g_i({\mathbf x})^\ast||_{2,\infty} \leq \\
\sum_{i=r+1}^{r+N} ||\sigma_i g^\ast_i||_{L_{\infty}} ||f_i||_{L_{2}} \leq 
\sum_{i=r+1}^{r+N} \sigma_i C_F 
\end{split}
\end{equation*}
Therefore, $F_r (D, {\mathbf x})$ is a Cauchy sequence in the Banach space of bounded operators ${\mathcal B}(L_{2}({\mathbb R}^n), L_{\infty}({\mathbb R}^n))$ with the operator norm. Thus, $F_r (D, {\mathbf x}) \stackrel{r\to\infty}{\to} F(D, {\mathbf x})$ with respect to the norm $||\cdot||_{2,\infty}$.

Moreover, the norm of the remaining part is bounded:
\begin{equation*}\label{expr0}
\begin{split}
||F({\mathbf x}, D)-F_r({\mathbf x}, D)||_{2,\infty}\leq 
C_F \sum_{i=r+1}^\infty \sigma_{i} 
\end{split}
\end{equation*}

By construction ${\rm dim}\, \mathcal{R} [O^\dag_{F_r}]=r$, i.e. $(F, F_r)\in \Pi(K,r)$. Therefore,
\begin{equation*}
\begin{split}
d_{K} (r) \leq ||F({\mathbf x}, D)-F_r({\mathbf x}, D)||_{2,\infty}\leq 
C_F \sum_{i=r+1}^\infty \sigma_{i} 
\end{split}
\end{equation*}
\end{proof}

Sometimes $C_F=\infty$ or $\sum_{i=1}^\infty \sigma_i$ diverges, which makes the last proposition useless. 
Still some guarantees can be given, though we have to change the norm in the definition of $d_K(r)$ to the Hilbert-Schmidt norm:
\begin{equation}
d^0_{K} (r) = \inf\limits_{(F,G)\in \Pi(K,r)} ||F(D, {\mathbf x})-G(D, {\mathbf x})||_{\rm HS}
\end{equation}
To obtain a bound on $d^0_{K} (r)$ we need a simple observation.
\begin{proposition}\label{abc} Let $F\in S^m_{1,0}$, then $$||{\rm O}_F||_{\rm HS} = ||F({\mathbf x}, D)||_{\rm HS}.$$
\end{proposition}
\begin{proof}
This fact follow directly from proposition~\ref{diag}. Since ${\rm O}_F{\rm O}^\dag_F$ has the same diagonal as $F({\mathbf x}, D)F({\mathbf x}, D)^\dag$, we conclude that ${\rm Tr\,}({\rm O}_F{\rm O}^\dag_F) = {\rm Tr\,}(F({\mathbf x}, D)F({\mathbf x}, D)^\dag)$. Therefore,  $||{\rm O}_F||_{\rm HS} = ||F({\mathbf x}, D)||_{\rm HS}$.
\end{proof}
\begin{proposition} Let $K$ be a Mercer kernel, $F\in {\rm Sym}(K)$ be a Hilbert-Schmidt kernel, $\sigma_{1}\geq \sigma_{2}\geq \cdots $ be singular values of the operator ${\rm O}_{F}$. Then, 
$$d^0_{K} (r) \leq \sqrt{\sum_{i=r+1}^\infty \sigma^2_i}.$$
\end{proposition}

\begin{proof}
Again, ${\rm dim}\, \mathcal{R} [O^\dag_{F_r}]=r$, i.e. $(F, F_r)\in \Pi(K,r)$. Therefore, using proposition~\ref{abc} we obtain:
\begin{equation*}
\begin{split}
d^0_{K} (r) \leq ||F(D,{\mathbf x})-F_r(D,{\mathbf x})||_{\rm HS} = \\
||O_{F}-O_{F_r}||_{\rm HS}\leq 
\sqrt{\sum_{i=r+1}^\infty \sigma^2_{i} }
\end{split}
\end{equation*}
\end{proof}
Unfortunately, many interesting kernels cannot be defined by Hilbert-Schmidt kernel $F$, but those cases can be partially captured by the following simple generalization whose proof is straightforward.
\begin{proposition}  Let $F$ be a kernel such that ${\rm dim}\, \mathcal{R} [O^\dag_{F}]=d$ and $\delta F$ be a Hilbert-Schmidt kernel such that $K=K^{F+\delta F}$ and $\sigma_{1}\geq \sigma_{2}\geq $ are singular values of the operator ${\rm O}_{\delta F}$. Then, for any $n\geq d$ we have
$$d_{K} (n) \leq c\sqrt{\sum_{i=n-d+1}^\infty \sigma^2_{i}}$$
where $c$ is a constant.
\end{proposition}

{\small
\bibliographystyle{plain}
\bibliography{lit}

\newcommand{\noopsort}[1]{}
\begin{thebibliography}{10}

\bibitem{pmlr-v70-arjovsky17a}
Martin Arjovsky, Soumith Chintala, and L{\'e}on Bottou.
\newblock {W}asserstein generative adversarial networks.
\newblock In Doina Precup and Yee~Whye Teh, editors, {\em Proceedings of the
  34th International Conference on Machine Learning}, volume~70 of {\em
  Proceedings of Machine Learning Research}, pages 214--223, International
  Convention Centre, Sydney, Australia, 06--11 Aug 2017. PMLR.

\bibitem{Rademacher}
Peter~L. Bartlett, Olivier Bousquet, and Shahar Mendelson.
\newblock Localized rademacher complexities.
\newblock In Jyrki Kivinen and Robert~H. Sloan, editors, {\em Computational
  Learning Theory}, pages 44--58, Berlin, Heidelberg, 2002. Springer Berlin
  Heidelberg.

\bibitem{binkowski2018demystifying}
Mikołaj Bińkowski, Danica~J. Sutherland, Michael Arbel, and Arthur Gretton.
\newblock Demystifying {MMD} {GAN}s.
\newblock In {\em International Conference on Learning Representations}, 2018.

\bibitem{Bubba}
Tatiana~A. Bubba, Mathilde Galinier, Matti Lassas, Marco Prato, Luca Ratti, and
  Samuli Siltanen.
\newblock Deep neural networks for inverse problems with pseudodifferential
  operators: An application to limited-angle tomography.
\newblock {\em SIAM Journal on Imaging Sciences}, 14(2):470--505, 2021.

\bibitem{Corinna}
Corinna Cortes, Marius Kloft, and Mehryar Mohri.
\newblock Learning kernels using local rademacher complexity.
\newblock In C.~J.~C. Burges, L.~Bottou, M.~Welling, Z.~Ghahramani, and K.~Q.
  Weinberger, editors, {\em Advances in Neural Information Processing Systems},
  volume~26. Curran Associates, Inc., 2013.

\bibitem{Rostamizadeh}
Corinna Cortes, Mehryar Mohri, and Afshin Rostamizadeh.
\newblock Generalization bounds for learning kernels.
\newblock In {\em Proceedings of the 27th International Conference on
  International Conference on Machine Learning}, ICML'10, page 247–254,
  Madison, WI, USA, 2010. Omnipress.

\bibitem{FELIUFABA2020109309}
Jordi Feliu-Fabà, Yuwei Fan, and Lexing Ying.
\newblock Meta-learning pseudo-differential operators with deep neural
  networks.
\newblock {\em Journal of Computational Physics}, 408:109309, 2020.

\bibitem{friedlander1998introduction}
F.G. Friedlander and M.S. Joshi.
\newblock {\em Introduction to the Theory of Distributions}.
\newblock Cambridge University Press, 1998.

\bibitem{ghiasi-shirazi10a}
Kamaledin Ghiasi-Shirazi, Reza Safabakhsh, and Mostafa Shamsi.
\newblock Learning translation invariant kernels for classification.
\newblock {\em Journal of Machine Learning Research}, 11(45):1353--1390, 2010.

\bibitem{Goodfellow}
Ian Goodfellow, Jean Pouget-Abadie, Mehdi Mirza, Bing Xu, David Warde-Farley,
  Sherjil Ozair, Aaron Courville, and Yoshua Bengio.
\newblock Generative adversarial nets.
\newblock In Z.~Ghahramani, M.~Welling, C.~Cortes, N.~Lawrence, and K.~Q.
  Weinberger, editors, {\em Advances in Neural Information Processing Systems},
  volume~27. Curran Associates, Inc., 2014.

\bibitem{MMD}
A.~Gretton, K.~Borgwardt, M.~Rasch, B.~Sch{\"o}lkopf, and A.~Smola.
\newblock A kernel two-sample test.
\newblock {\em Journal of Machine Learning Research}, 13:723--773, March 2012.

\bibitem{hormander1966pseudo}
L.~H{\"o}rmander.
\newblock {\em Pseudo-differential Operators and Hypoelliptic Equations}.
\newblock Institute for Advanced Study, 1966.

\bibitem{Kloft}
Marius Kloft and Gilles Blanchard.
\newblock The local rademacher complexity of lp-norm multiple kernel learning.
\newblock In J.~Shawe-Taylor, R.~Zemel, P.~Bartlett, F.~Pereira, and K.~Q.
  Weinberger, editors, {\em Advances in Neural Information Processing Systems},
  volume~24. Curran Associates, Inc., 2011.

\bibitem{pseudodifferential}
J.~J. Kohn and L.~Nirenberg.
\newblock An algebra of pseudo-differential operators.
\newblock {\em Communications on Pure and Applied Mathematics},
  18(1-2):269--305, 1965.

\bibitem{Koltchinskii}
Vladimir Koltchinskii and Ming Yuan.
\newblock {Sparsity in multiple kernel learning}.
\newblock {\em The Annals of Statistics}, 38(6):3660 -- 3695, 2010.

\bibitem{MMD-GAN}
Chun-Liang Li, Wei-Cheng Chang, Yu~Cheng, Yiming Yang, and Barnab\'{a}s
  P\'{o}czos.
\newblock Mmd gan: Towards deeper understanding of moment matching network.
\newblock In {\em Proceedings of the 31st International Conference on Neural
  Information Processing Systems}, NIPS'17, page 2200–2210, Red Hook, NY,
  USA, 2017. Curran Associates Inc.

\bibitem{gmmn}
Yujia Li, Kevin Swersky, and Richard Zemel.
\newblock Generative moment matching networks.
\newblock In {\em Proceedings of the 32nd International Conference on
  International Conference on Machine Learning - Volume 37}, ICML'15, page
  1718–1727. JMLR.org, 2015.

\bibitem{Mendelson03onthe}
Shahar Mendelson, Thore Graepel, and Ralf Herbrich.
\newblock On the performance of kernel classes.
\newblock {\em Journal of Machine Learning Research}, 4:2003, 2003.

\bibitem{Regularizationkernel}
Shahar Mendelson and Joseph Neeman.
\newblock {Regularization in kernel learning}.
\newblock {\em The Annals of Statistics}, 38(1):526 -- 565, 2010.

\bibitem{Mehryar}
Mehryar Mohri and Afshin Rostamizadeh.
\newblock Rademacher complexity bounds for non-i.i.d. processes.
\newblock In D.~Koller, D.~Schuurmans, Y.~Bengio, and L.~Bottou, editors, {\em
  Advances in Neural Information Processing Systems}, volume~21. Curran
  Associates, Inc., 2009.

\bibitem{KernelMean}
Krikamol Muandet, Kenji Fukumizu, Bharath Sriperumbudur, and Bernhard
  Schölkopf.
\newblock Kernel mean embedding of distributions: A review and beyond.
\newblock {\em Foundations and Trends® in Machine Learning}, 10(1-2):1--141,
  2017.

\bibitem{NIPS2016_cedebb6e}
Sebastian Nowozin, Botond Cseke, and Ryota Tomioka.
\newblock f-gan: Training generative neural samplers using variational
  divergence minimization.
\newblock In D.~Lee, M.~Sugiyama, U.~Luxburg, I.~Guyon, and R.~Garnett,
  editors, {\em Advances in Neural Information Processing Systems}, volume~29.
  Curran Associates, Inc., 2016.

\bibitem{potter2021parameterized}
Kevin~M. Potter, Steven~Richard Sleder, Matthew~David Smith, and John Tencer.
\newblock Parameterized pseudo-differential operators for graph convolutional
  neural networks, 2021.

\bibitem{Schrohe1986}
Elmar Schrohe.
\newblock Complex powers of elliptic pseudodifferential operators.
\newblock {\em Integral Equations and Operator Theory}, 9(3):337--354, May
  1986.

\bibitem{HilbertSpaceEmbeddings}
Bharath~K. Sriperumbudur, Arthur Gretton, Kenji Fukumizu, Bernhard
  Sch\"{o}lkopf, and Gert~R.G. Lanckriet.
\newblock Hilbert space embeddings and metrics on probability measures.
\newblock 11:1517–1561, August 2010.

\bibitem{Steinwart}
Ingo Steinwart and Andreas Christmann.
\newblock {\em Support Vector Machines}.
\newblock Springer Publishing Company, Incorporated, 1st edition, 2008.

\bibitem{Steinwart2012}
Ingo Steinwart and Clint Scovel.
\newblock Mercer's theorem on general domains: On the interaction between
  measures, kernels, and rkhss.
\newblock {\em Constructive Approximation}, 35(3):363--417, Jun 2012.

\bibitem{Taylor2017}
Michael~Eugene Taylor.
\newblock {\em Pseudodifferential Operators (PMS-34)}.
\newblock Princeton University Press, 2017.

\bibitem{Zhou}
Ding-Xuan Zhou.
\newblock The covering number in learning theory.
\newblock {\em J. Complex.}, 18(3):739–767, September 2002.

\end{thebibliography}
}

\end{document}